\documentclass{article}

\usepackage[main, final]{neurips_2026}
\usepackage[utf8]{inputenc}
\usepackage[T1]{fontenc}
\usepackage{hyperref}
\usepackage{url}
\usepackage{booktabs}
\usepackage{needspace}
\usepackage{amsfonts}
\usepackage{amsmath}
\usepackage{amssymb}
\usepackage{amsthm}
\usepackage{thm-restate}
\usepackage{nicefrac}
\usepackage{microtype}
\usepackage{xcolor}
\usepackage{tcolorbox}
\usepackage{graphicx}
\usepackage{pgfplots}
\pgfplotsset{compat=1.17}

\newtheorem{theorem}{Theorem}[section]

\newtheorem{lemma}[theorem]{Lemma}

\title{Oversmoothing as Representation Degeneracy in Neural Sheaf Diffusion}

\author{%
Arif D\"onmez$^{1,2,*}$ \quad Axel Mosig$^{2,3}$ \quad Ellen Fritsche$^{1,2,4}$ \quad Katharina Koch$^{1,2}$\\
$^1$ IUF -- Leibniz Research Institute for Environmental Medicine, D\"usseldorf, Germany\\
$^2$ DNTOX GmbH, D\"usseldorf, Germany\\
$^3$ Bioinformatics Group, Ruhr University Bochum, Bochum, Germany\\
$^4$ Swiss Centre for Applied Human Toxicology (SCAHT), Basel, Switzerland\\
$^*$ Corresponding author: \texttt{arif.doenmez@ruhr-uni-bochum.de}
}

\begin{document}

\maketitle

\begin{abstract}
Neural Sheaf Diffusion (NSD) learns a task-adapted sheaf geometry for node classification. We connect the study of these learned geometries with Geometric Invariant Theory (GIT) to analyse their representation-theoretic structure and develop mechanisms for geometric control during training. Identifying cellular sheaves with representations of the incidence quiver connects direct-sum decompositions to the harmonic space and motivates the study of trivial or low-complexity summands as potential degeneration modes. We make this perspective concrete by focusing on the trivial representation: on an embedded copy, the sheaf Dirichlet energy reduces to weighted scalar graph energy. The GIT stability framework provides a criterion for excluding such copies within the semistable class. Using its moment-map characterisation of polystability, we construct differentiable surrogates that encourage the corresponding balance conditions. We identify an equal-stalk obstruction to assigning a destabilising weight to the trivial representation and show how nonuniform stalk dimensions remove it. Controlled energy-minimisation experiments reproduce trivial-component formation, while native feature-conditioned NSD experiments on heterophilic graphs show how moment-map regularisation changes the balance and Dirichlet-energy trajectories of learned geometries during training, alongside their classification behaviour.
\end{abstract}

\section{Introduction}
\label{sec:intro}

Many Graph Neural Networks (GNNs) can be understood as discrete diffusion processes on graphs \citep{kipf2017semi, velickovic2018graph}. In message-passing architectures, nodes repeatedly aggregate information from their neighbors, effectively implementing a form of Laplacian smoothing. While this diffusion perspective explains the success of GNNs in propagating local information, it also exposes a fundamental limitation: repeated smoothing can drive node representations toward indistinguishable limits. This phenomenon, known as oversmoothing, is especially problematic in deep architectures where node features collapse to low-dimensional or nearly constant subspaces \citep{li2018deeper, chen2020measuring}. In the classical unnormalized graph Laplacian setting, this collapse is a structural consequence of the underlying diffusion geometry: the heat equation asymptotically projects features onto the kernel of the Laplacian, which for a connected graph consists of constant signals. However, this tendency toward identical node representations can be mismatched for heterophilic graphs, motivating the search for richer geometric structures that avoid collapsing into overly simple harmonic states.

Neural Sheaf Diffusion (NSD) addresses this by lifting the diffusion process into a vector bundle-like structure, replacing the standard Laplacian with a sheaf Laplacian whose restriction maps are learned from the data \citep{hansen2019toward, bodnar2022neural}. By learning a task-adapted geometry, NSD allows the network to model heterophily: the space of global sections, which governs the limit of fixed linear sheaf diffusion, can encode complex, signed, and multi-dimensional feature agreements rather than simple constants. However, optimizing these highly parameterized general sheaves can be challenging \citep{barbero2022sheafattention, barbero2022connection}. Furthermore, gradient-based learning exhibits implicit biases \citep{rahaman2019spectral, shah2020pitfalls, arora2019implicit}. This motivates investigating whether unregularized learning can favor overly simple restriction maps and thereby limit the expressive geometry available to diffusion.

In this work, we study this potential failure mode through the lens of quiver representation theory \citep{assem2006elements, schiffler2014quiver} and its geometric interpretation \citep{bongartz1998some, kraft1982geometric,derksen2005quiver, reineke2008modulirepresentationsquivers}. Identifying cellular sheaves on graphs with representations of the associated incidence quiver realizes the learned restriction maps as points in a finite-dimensional representation space. This provides a representation-theoretic perspective on asymptotic sheaf diffusion: direct-sum decompositions of the sheaf representation induce corresponding decompositions of its harmonic space. We investigate trivial or low-complexity summands as potential degeneration modes, focusing on the trivial representation, $\mathcal{F}_{\mathrm{triv}}$, on whose embedded copies the sheaf Dirichlet energy reduces to weighted scalar graph energy.

To control such potential degeneration during learning, we invoke Geometric Invariant Theory (GIT) \citep{kraft1985geometrische, popov1994invariant, mumford1994geometric}, particularly King's stability framework for quiver representations \citep{king1994moduli}. Its stability framework provides a criterion for excluding the trivial representation as a subrepresentation when its stability weight is negative. The moment-map characterization of polystability motivates differentiable regularizers that encourage the corresponding balance conditions in the learned restriction maps. This yields a conditional exclusion mechanism: if the learned geometry satisfies the corresponding polystability conditions, a negative weight rules out the trivial subrepresentation. Crucially, our algebraic analysis identifies a structural obstruction in equal-stalk architectures: when vertex and edge stalks have equal dimensions ($d_v = d_e$), admissibility forces the weight of the trivial all-object representation to vanish, placing it on a stability wall. Nonuniform stalk dimensions remove this obstruction, allowing adaptive stability parameters to assign it a destabilizing weight.

Our experiments examine this mechanism through controlled energy minimization, recovery from collapsed initial configurations, and native feature-conditioned NSD training. Controlled Dirichlet-energy minimization demonstrates the formation of trivial direct summands, while architectural comparisons confirm that rectangular stalks allow learned stability weights to leave the wall imposed by equal stalk dimensions. In native NSD training, moment-map regularization steers the balance and Dirichlet-energy trajectories of the learned geometries, with effects depending on the regularizer, dataset and network depth. Recovery experiments investigate learning from explicitly constructed trivial components with initially collapsed node signals. Under increased classification difficulty and on the WebKB graphs studied, these experiments provide evidence of improved held-out accuracy under regularization, including at fixed training endpoints where both regularized and unregularized models achieve complete training accuracy. Together, these findings illustrate geometric control during learning and setting-dependent generalization benefits, without establishing that the task improvements result from removal of trivial components or constitute a general prevention of oversmoothing.

\textbf{Contributions.} Our main contribution is to connect the study of Neural Sheaf Diffusion with Geometric Invariant Theory. Through a quiver-theoretic interpretation of learned sheaf geometries, stability, polystability and moment maps provide both a finer analysis of potential degeneration modes and concrete mechanisms for geometric control during learning. We develop differentiable regularizers and identify architectural conditions that enable this control, with experiments demonstrating their influence on learned geometries.

\section{Background}

\textbf{Graph Diffusion and Oversmoothing.} Let $G=(V,E)$ be an undirected graph with $n=|V|$ nodes, and let $\mathbf{X}\in\mathbb{R}^{n\times d}$ denote a node signal matrix. Standard graph diffusion is governed by the heat equation $\dot{\mathbf{X}}(t)=-\Delta\mathbf{X}(t)$, where $\Delta$ is the normalized graph Laplacian. As $t\to\infty$, the signal converges to the orthogonal projection onto $\ker(\Delta)$. For a connected graph, the scalar graph Laplacian has a one-dimensional kernel corresponding to globally constant signals, up to the usual degree weighting in the symmetrically normalized case. Thus, fixed linear diffusion suppresses all signal components orthogonal to this kernel, providing the basic smoothing mechanism underlying oversmoothing analyses in GNNs \citep{li2018deeper,chen2020measuring}.

\textbf{Cellular Sheaves.} A \emph{cellular sheaf} $\mathcal{F}$ over a graph $G$ equips the network with a richer structure by assigning a vector space $\mathcal{F}(v)$ to each vertex $v\in V$, a vector space $\mathcal{F}(e)$ to each edge $e\in E$, and a linear restriction map $\mathcal{F}_{v\trianglelefteq e}:\mathcal{F}(v)\to\mathcal{F}(e)$ for every incidence $v\trianglelefteq e$ \citep{curry2014sheaves,shepard1985cellular}. The space of $0$-cochains, or vertex-level sheaf signals, is $C^0(G;\mathcal{F})=\bigoplus_{v\in V}\mathcal{F}(v)$, and the space of $1$-cochains is $C^1(G;\mathcal{F})=\bigoplus_{e\in E}\mathcal{F}(e)$. Choosing an arbitrary edge orientation $e=(u,v)$, the coboundary operator $\delta_{\mathcal{F}}:C^0(G;\mathcal{F})\to C^1(G;\mathcal{F})$ is defined component-wise as $(\delta_{\mathcal{F}}\mathbf{x})_e = \mathcal{F}_{v\trianglelefteq e}\mathbf{x}_v - \mathcal{F}_{u\trianglelefteq e}\mathbf{x}_u$. It measures the disagreement of neighboring vertex data after both have been transported into the corresponding edge stalk. Equipping the stalks with inner products induces inner products on the cochain spaces.

\textbf{Sheaf Laplacians and NSD.} The sheaf Laplacian is defined as $\Delta_{\mathcal{F}}=\delta_{\mathcal{F}}^\ast\delta_{\mathcal{F}}$ \citep{hansen2019toward}. It is positive semidefinite, and its kernel coincides exactly with the space of \emph{global sections} $H^0(G;\mathcal{F})=\ker(\delta_{\mathcal{F}})=\ker(\Delta_{\mathcal{F}})$. A global section represents an assignment of local vertex data that is perfectly consistent across all edges after transport through the sheaf maps. Neural Sheaf Diffusion (NSD) generalizes classical graph diffusion by learning feature-conditioned restriction maps through a parameterized sheaf learner \citep{bodnar2022neural}. For a fixed sheaf geometry, the linear diffusion equation $\dot{\mathbf{X}}(t)=-\Delta_{\mathcal{F}}\mathbf{X}(t)$ converges asymptotically to $\Pi_{H^0(G;\mathcal{F})}\mathbf{X}(0)$. This identifies global sections as the limiting signal space of fixed linear sheaf diffusion and motivates studying their algebraic structure. For an invertible positive block normalization $N$, as used in our implementation, the normalized operator $N\Delta_{\mathcal{F}}N$ instead has kernel $N^{-1}H^0(G;\mathcal{F})$; the harmonic coordinates are transported accordingly (Appendix~\ref{app:experiments}).

\textbf{Quiver Representations.} To analyze this algebraic structure, we use quiver representation theory \citep{derksen2005quiver, reineke2008modulirepresentationsquivers, schiffler2014quiver}. A quiver $Q=(Q_0,Q_1)$ consists of vertices $Q_0$ and directed arrows $Q_1$. A representation $M$ of $Q$ assigns a vector space $M_i$ to each $i\in Q_0$ and a linear map $M_a:M_{s(a)}\to M_{t(a)}$ to each arrow $a\in Q_1$, where $s(a)$ and $t(a)$ denote source and target. For a fixed dimension vector $\mathbf{d}=(\dim M_i)_{i\in Q_0}$, the space of all such representations is $\operatorname{Rep}(Q,\mathbf{d}) = \prod_{a\in Q_1} \operatorname{Hom}\left(k^{d_{s(a)}},k^{d_{t(a)}}\right)$. The base-change, or gauge, group $G_{\mathbf{d}}=\prod_{i\in Q_0}GL(d_i)$ acts on this space by changing bases in the assigned vector spaces. We use the identification of cellular sheaves on a graph with representations of its incidence quiver to connect learned sheaf geometries with direct-sum decompositions, stability and moduli of quiver representations.

\section{Cellular Sheaves as Quiver Representations}
\label{sec:quiver_sheaves}

To analyze the representation geometry of Neural Sheaf Diffusion, we first make explicit the dictionary between cellular sheaves on a graph and quiver representations. This correspondence identifies the learnable restriction maps of an NSD model with representation data in a finite-dimensional algebraic space.

\textbf{The Incidence Quiver.} For an undirected graph $G=(V,E)$, define its \emph{incidence quiver} $Q_G=(Q_0,Q_1)$ as the bipartite quiver with vertex set $Q_0=V\sqcup E$. For every incidence $v\trianglelefteq e$, we introduce one directed arrow $a_{v,e}:v\to e$. Thus, $Q_1=\{a_{v,e}:v\to e \mid v\trianglelefteq e\}$. The quiver $Q_G$ should not be confused with an orientation of the original graph: graph edges become objects of the quiver, and each such edge-object receives arrows from its incident endpoints. For a graph edge $e=\{u,v\}$, the incidence quiver therefore contains two arrows $u\to e$ and $v\to e$. All arrows run from vertex-objects to edge-objects; consequently, $Q_G$ is acyclic, even when the underlying graph contains cycles.
\[
\underbrace{u\mathrel{\text{---}}v}_{G:\ e=\{u,v\}}
\quad\longmapsto\quad
\underbrace{u\longrightarrow e\longleftarrow v}_{Q_G}
\quad\longmapsto\quad
\underbrace{\mathcal{F}(u)\xrightarrow{A_{u,e}}\mathcal{F}(e)
\xleftarrow{A_{v,e}}\mathcal{F}(v)}_{\text{stalks and restriction maps}}.
\]

\textbf{The Correspondence.} A representation $M$ of $Q_G$ assigns a vector space $M_v$ to each graph vertex $v\in V$, a vector space $M_e$ to each graph edge $e\in E$, and a linear map $M_{a_{v,e}}:M_v\to M_e$ to each incidence arrow. This is exactly the data of a cellular sheaf $\mathcal{F}$ on $G$: we set $M_v=\mathcal{F}(v)$, $M_e=\mathcal{F}(e)$, and $M_{a_{v,e}}=\mathcal{F}_{v\trianglelefteq e}$. Conversely, every representation of $Q_G$ defines a cellular sheaf by the same assignment. Hence the category of finite-dimensional cellular sheaves on $G$ is equivalent to the category of finite-dimensional representations of the incidence quiver $Q_G$.

\textbf{Dimension Vectors and Parameter Space.} In Neural Sheaf Diffusion, the stalk dimensions are architectural choices. Writing $d_v=\dim\mathcal{F}(v)$ and $d_e=\dim\mathcal{F}(e)$ gives a dimension vector $\mathbf{d}=\big((d_v)_{v\in V},(d_e)_{e\in E}\big)$. For this fixed dimension vector, the space of cellular sheaves with these stalk dimensions is the quiver representation space $\operatorname{Rep}(Q_G,\mathbf{d}) = \prod_{v\trianglelefteq e} \operatorname{Hom}\left(k^{d_v},k^{d_e}\right)$. For each layer and forward pass, the feature-conditioned sheaf learner produces a point in $\operatorname{Rep}(Q_G,\mathbf{d})$. Training updates the learner parameters and thereby moves the generated restriction maps through the subset of representation space accessible to this parameterization.

\textbf{Gauge Equivalence and Moduli.} The matrix representation of a sheaf depends on choices of bases in the stalks. Changing bases by $g_v\in GL(d_v)$ and $g_e\in GL(d_e)$ transforms each restriction map by $\mathcal{F}_{v\trianglelefteq e} \longmapsto g_e\mathcal{F}_{v\trianglelefteq e}g_v^{-1}$. This is precisely the base-change action of the gauge group $G_{\mathbf{d}} = \prod_{v\in V}GL(d_v) \times \prod_{e\in E}GL(d_e)$ on $\operatorname{Rep}(Q_G,\mathbf{d})$. Thus, gauge orbits classify the underlying sheaf representations up to isomorphism. The diffusion operator additionally depends on the inner products on the stalks, which must be transported consistently under changes of bases. This leads naturally to a moduli problem for the underlying representations, for which GIT provides stability-dependent quotient constructions.

\section{The Kernel Decomposition Theorem}
\label{sec:kernel_decomposition}

Having established the equivalence between cellular sheaves and representations of the incidence quiver $Q_G$, we can recast the diffusion limit in the language of homological algebra.

\textbf{Global Sections as Categorical Limits.} For a cellular sheaf $\mathcal{F}$ viewed as a representation diagram $M:Q_G\to\mathbf{Vect}$, the space of global sections is canonically isomorphic to the categorical limit of this diagram. Indeed, the limit consists of compatible assignments to all vertex and edge stalks, while $H^0(G;\mathcal{F})=\ker(\delta_{\mathcal{F}})$ records the corresponding vertex assignments satisfying the compatibility equations. Thus the operation sending a sheaf to its harmonic space is the global section functor $\Gamma:\operatorname{Rep}(Q_G)\to\mathbf{Vect}$, $\mathcal{F}\longmapsto H^0(G;\mathcal{F})$. Equivalently, $\Gamma$ is the limit functor applied to the incidence-quiver representation. As a limit functor, it is left-exact; in particular, it preserves finite limits.

\textbf{Functorial Decomposition.} By the Krull--Schmidt theorem, every finite-dimensional representation in the abelian category $\operatorname{Rep}(Q_G)$ admits a direct-sum decomposition into indecomposable objects, unique up to isomorphism and permutation: $\mathcal{F} \cong \bigoplus_{k=1}^K \mathcal{F}^{(k)}$ \citep{assem2006elements}. Since finite direct sums in $\mathbf{Vect}$ are biproducts, and since $\Gamma$ preserves finite products, the global section functor preserves this decomposition up to natural isomorphism. Applying $\Gamma$ gives $H^0(G;\mathcal{F}) \cong \bigoplus_{k=1}^K H^0\big(G;\mathcal{F}^{(k)}\big)$. Thus the decomposition of the diffusion limit does not require a separate spectral argument or manual block-diagonalization of the sheaf Laplacian. It is a functorial consequence of viewing learned cellular sheaves as objects of the incidence-quiver representation category.

\textbf{Homological Interpretation of Oversmoothing.}
For a fixed sheaf geometry, continuous linear diffusion asymptotically projects node features onto the global-section space $H^0(G;\mathcal{F})$. Since the global section functor preserves finite direct sums, this limiting signal space decomposes algebraically into the harmonic contributions of the individual Krull--Schmidt summands. This connects asymptotic diffusion to representation structure and motivates investigating trivial or low-complexity summands as potential degeneration modes. Their effect on task-relevant information depends on their embeddings in the vertex stalks and on how the input features project onto the resulting harmonic space. We use this homological perspective to motivate geometric control of learned representations through stability and moment-map regularization.

\textbf{Subrepresentations and Harmonic Injection.} While the Krull--Schmidt theorem describes direct-sum decompositions, the relationship between subrepresentations and global sections is governed by the exactness properties of the limit functor.

\begin{lemma}[Harmonic Injection]
Let $\mathcal{F}$ be a cellular sheaf on $G$ and let $\mathcal{F}' \hookrightarrow \mathcal{F}$ be a subrepresentation of the corresponding incidence quiver. Then there is an induced injection of global sections $H^0(G;\mathcal{F}') \hookrightarrow H^0(G;\mathcal{F})$.
\end{lemma}
\begin{proof}
The global section functor $\Gamma: \operatorname{Rep}(Q_G) \to \mathbf{Vect}$ acts by taking the categorical limit of the representation diagram. Because the limit functor is left exact, it preserves monomorphisms. Therefore, the inclusion of the subrepresentation $\mathcal{F}' \hookrightarrow \mathcal{F}$ induces an injective linear map $\Gamma(\mathcal{F}') \hookrightarrow \Gamma(\mathcal{F})$.
\end{proof}

This extends the decomposition perspective beyond direct summands: subrepresentations contribute injectively to the global-section space even when they do not split off. The next section uses this observation to motivate stability-based control of potential degeneration modes.

\section{Representation Stability and Oversmoothing}
\label{sec:stability}

Having developed a representation-theoretic perspective on potential oversmoothing mechanisms, we now turn to stability and moduli as a framework for studying and controlling learned sheaf geometries.

\textbf{The Trivial Subrepresentation.} To understand why representation degeneracy relates to oversmoothing, consider the \emph{trivial representation} $\mathcal{F}_{\mathrm{triv}}$, which assigns a one-dimensional space to every vertex and edge and the identity map to every incidence. If $\mathcal{F}_{\mathrm{triv}}$ exists as a subrepresentation of a learned sheaf $\mathcal{F}$, there exists a 1-dimensional subspace $W_v \subset \mathbb{R}^{d_v}$ at every vertex $v$ and a 1-dimensional subspace $W_e \subset \mathbb{R}^{d_e}$ at every edge $e$ such that the restriction maps act as an isomorphism between these subspaces. Since this subrepresentation is isomorphic to $\mathcal{F}_{\mathrm{triv}}$, we may choose compatible bases $w_v \in W_v$ and $w_e \in W_e$ so that $\mathcal{F}_{v \trianglelefteq e} w_v = w_e$. For any signal $\mathbf{x}$ restricted to this subrepresentation, the node features take the form $\mathbf{x}_v = c_v w_v$ for some scalars $c_v \in \mathbb{R}$. The sheaf Dirichlet energy along an edge $e = (u,v)$ then evaluates to:
$$
\begin{aligned}
    \left\| \mathcal{F}_{v \trianglelefteq e} \mathbf{x}_v - \mathcal{F}_{u \trianglelefteq e} \mathbf{x}_u \right\|^2
    &= \left\| c_v (\mathcal{F}_{v \trianglelefteq e} w_v) - c_u (\mathcal{F}_{u \trianglelefteq e} w_u) \right\|^2 \\
    &= \left\| c_v w_e - c_u w_e \right\|^2 \\
    &= (c_v - c_u)^2 \|w_e\|^2.
\end{aligned}
$$
Thus, on the embedded trivial copy, the sheaf Dirichlet energy reduces to weighted scalar graph energy. For a connected graph, its global sections have constant coefficients $c_v$ in the compatible bases.

\textbf{Optimization Bias and Collapse.} Implicit biases toward simple solutions provide a motivation for studying low-complexity representation structures in learned sheaf geometries \citep{rahaman2019spectral,gunasekar2017implicit,shah2020pitfalls}. The trivial representation offers a concrete candidate: signals with constant coefficients in its compatible bases have zero Dirichlet energy. If learning favors such components and the node signals concentrate in them without preserving task-relevant distinctions, scalar-like smoothing can dominate the resulting representation. We investigate this as a potential degeneration mechanism rather than a universal consequence of classification training.

\textbf{Geometric Intuition for Degeneration.} Quiver representation theory provides a geometric language for degeneration through nonclosed gauge orbits and their boundary representations \citep{bongartz1998some,reineke2008modulirepresentationsquivers}. A useful mental model in our setting is a learned sheaf whose decomposition becomes dominated by copies of $\mathcal{F}_{\mathrm{triv}}$, with $\bigoplus\mathcal{F}_{\mathrm{triv}}$ as an idealized limiting configuration. On each trivial copy, the sheaf Dirichlet energy reduces to weighted scalar graph energy, and harmonicity requires constant coefficients across a connected graph. If the learned signals concentrate in these components without retaining task-relevant distinctions, the richer sheaf geometry effectively supports only scalar-like smoothing. This motivates studying trivial-component formation and mechanisms for controlling the learned representation geometry, without assuming that every trivial component is harmful or that classification training necessarily converges to this configuration. Our experiments in Section~\ref{sec:experiments} examine this perspective through controlled trivial-component formation, recovery from deliberately collapsed initial configurations, and geometric regularization during native NSD training. These complementary settings illustrate how the proposed representation-theoretic viewpoint informs both the analysis and control of learned sheaf geometries.

\textbf{Representation Degeneracy.} By the kernel decomposition theorem, any direct-sum decomposition $\mathcal{F}\cong\bigoplus_k\mathcal{F}^{(k)}$ induces a decomposition $H^0(G;\mathcal{F})\cong\bigoplus_k H^0(G;\mathcal{F}^{(k)})$. This makes the harmonic contributions of individual summands explicit. We use \emph{representation degeneracy} to describe configurations in which the surviving signal is dominated by contributions that fail to retain task-relevant distinctions. Trivial or low-complexity summands provide concrete candidates for this mechanism, although their presence alone does not imply an uninformative diffusion limit.

\textbf{King Stability.} King stability provides an algebraic way to exclude certain subrepresentations from a representation class \citep{king1994moduli}. Fix a stability parameter $\theta\in\mathbb{R}^{Q_0}$ satisfying $\theta\cdot\dim\mathcal{F}=0$. We use the convention that $\mathcal{F}$ is $\theta$-semistable if every nonzero proper subrepresentation $\mathcal{F}'\subset\mathcal{F}$ satisfies $\theta\cdot\dim\mathcal{F}'\ge 0$. Within the $\theta$-semistable class, choosing $\theta(\mathcal{F}_{\mathrm{triv}})<0$ therefore excludes the trivial representation as a subrepresentation, without requiring it to occur as a direct summand. The negative weight alone imposes no such restriction on an arbitrary learned representation. This criterion motivates the moment-map regularizers introduced below, which encourage the balance conditions associated with polystability.

King stability controls subrepresentations directly, whereas the kernel decomposition theorem describes the harmonic contributions of direct summands. Since every direct summand is also a subrepresentation, the stability criterion excludes trivial direct summands as a special case. Together with harmonic injection, these results connect representation structure to the diffusion limit and motivate stability-based geometric control during learning.

\textbf{Caveats of Degeneration.} We emphasize that excluding $\mathcal{F}_{\mathrm{triv}}$ prevents only one specific mode of homological collapse. A learned sheaf may also degenerate into a geometry where $H^0(G;\mathcal{F}) = 0$. In this scenario, all signals decay to zero under continuous linear diffusion with the sheaf geometry held fixed. While $\theta$-stability provides a targeted tool to mathematically exclude the trivial constant-signal mode, it does not universally guarantee that the surviving harmonic space is highly expressive or perfectly aligned with the downstream task.

\section{Moment-Map Regularization and Architectural Constraints}
\label{sec:moment_map}

The exclusion criterion above requires the learned sheaf to be $\theta$-semistable. Enforcing this definition directly would require controlling all subrepresentations, that is, all families of stalk subspaces compatible with the restriction maps. This does not provide a convenient differentiable training objective. We therefore turn to $\theta$-polystability, a stronger condition that admits a moment-map characterization.

\textbf{From Semistability to Polystability.}
A representation is $\theta$-stable if it has total weight zero and every nonzero proper subrepresentation has strictly positive weight, using our sign convention. It is $\theta$-polystable if it is a direct sum of $\theta$-stable representations, each of weight zero. Every $\theta$-polystable representation is therefore $\theta$-semistable. Consequently, polystability is sufficient for our exclusion mechanism: a $\theta$-polystable sheaf with $\theta(\mathcal{F}_{\mathrm{triv}})<0$ cannot contain $\mathcal{F}_{\mathrm{triv}}$ as a subrepresentation.

\textbf{The Moment-Map Criterion.}
For this discussion, we use the standard complexified setting over $k=\mathbb{C}$. Choose Hermitian inner products on all stalks. The compact gauge group
$K_{\mathbf d}=\prod_{i\in V\sqcup E}U(d_i)$
then acts unitarily on the representation space. Writing
$A_{v,e}:=\mathcal{F}_{v\trianglelefteq e}$,
we use the moment-map components
\[
\mu_v(\mathcal{F})
=
-\sum_{e:\,v\trianglelefteq e}A_{v,e}^{\ast}A_{v,e},
\qquad
\mu_e(\mathcal{F})
=
\sum_{v:\,v\trianglelefteq e}A_{v,e}A_{v,e}^{\ast}.
\]
These signs agree with the semistability convention used above.

King's moment-map criterion \citep[Proposition~6.5]{king1994moduli}, building on the Kempf--Ness correspondence \citep{kempf1979length}, states that
\[
\mathcal{F}\text{ is }\theta\text{-polystable}
\quad\Longleftrightarrow\quad
\exists\,g\in G_{\mathbf d}:
\ \mu_i(g\cdot\mathcal{F})=\theta_i I_{d_i}
\quad\text{for every }i\in V\sqcup E.
\]
Here $G_{\mathbf d}=\prod_i GL(d_i,\mathbb{C})$ is the complex gauge group. Equivalently, the balance equations hold for $\mathcal{F}$ after a suitable choice of stalk inner products.\footnote{King's algebraic GIT construction uses integral stability weights. Our learning surrogate treats the weights as real-valued parameters in the same balance equations.}

This criterion replaces conditions on subrepresentations by matrix balance equations. We use these equations to construct differentiable penalties on the raw restriction maps, with the training inner products held fixed. Exact balance would imply polystability; a small finite penalty is a surrogate and does not by itself certify polystability or exclusion. The real-valued implementation follows the same matrix formulas, replacing Hermitian adjoints by transposes.

\textbf{Central Moment Regularization (CentMM).}
The balance equations require each moment-map block to be a scalar multiple of the identity. CentMM encourages this form without prescribing the scalar in advance. For a given block $\mu_i$, its closest scalar matrix in Frobenius norm is $\widehat{\theta}_i I_{d_i}$, where
\[
\widehat{\theta}_i
=
\frac{\operatorname{tr}(\mu_i)}{d_i}.
\]
We therefore define
\[
\mathcal{R}_{\mathrm{cent}}(\mathcal{F})
=
\sum_i
\left\|
\mu_i(\mathcal{F})
-
\widehat{\theta}_i I_{d_i}
\right\|_F^2.
\]
Every incidence map contributes opposite traces at its source and target, so
$\sum_i\operatorname{tr}(\mu_i)=0$.
The implicit weights consequently satisfy
$\sum_i d_i\widehat{\theta}_i=0$.
Thus, CentMM encourages the form of the polystable balance equations while selecting the scalar targets analytically from the current geometry. Its training objective is
$\mathcal{L}=\mathcal{L}_{\mathrm{task}}+\lambda_\mu\mathcal{R}_{\mathrm{cent}}$.
It acts as a geometric balance bias, without imposing a prescribed negative weight on the trivial representation.

\textbf{Adaptive Stability and Implementation (ThetaMM).}
ThetaMM instead retains an explicit learnable stability weight $\theta_i$ for each graph object. The quadratic balance residual is
\[
\mathcal{R}_{\theta\text{-}\mu}(\mathcal{F};\theta)
=
\sum_i
\left\|
\mu_i(\mathcal{F})-\theta_i I_{d_i}
\right\|_F^2.
\]
For this quadratic residual alone, CentMM is precisely its analytic minimization over admissible $\theta$:
\[
\mathcal{R}_{\mathrm{cent}}(\mathcal{F})
=
\min_{\theta\cdot\mathbf d=0}
\mathcal{R}_{\theta\text{-}\mu}(\mathcal{F};\theta).
\]
Keeping $\theta$ explicit allows us to introduce additional constraints on its weights. This analytic equivalence does not generally extend to the Huber variant or to objectives with an additional $\theta$-norm penalty.

We enforce admissibility by projecting unconstrained parameters $\widetilde{\theta}$ onto the orthogonal complement of the dimension vector:
\[
\theta
=
\widetilde{\theta}
-
\frac{\widetilde{\theta}\cdot\mathbf d}
{\|\mathbf d\|_2^2}\mathbf d.
\]
This projection is differentiable and ensures $\theta\cdot\mathbf d=0$. In the negative-weight experiments, we additionally enforce
$\theta(\mathcal{F}_{\mathrm{triv}})\leq-\varepsilon$
for a fixed $\varepsilon>0$ after each optimization step. The experimental implementation uses averaged residuals, including a Huber variant for ThetaMM and an additional $\theta$-norm penalty, as specified in Appendix~\ref{app:experiments}.

\textbf{Architectural Constraint.}
If all stalks have dimension $d$, admissibility gives
$0=\theta\cdot\mathbf d=d\,\theta(\mathcal F_{\mathrm{triv}})$.
Thus, equal-stalk architectures place the trivial representation
on a stability wall. Nonuniform dimensions remove this
proportionality obstruction, making negative weights admissible
in the rectangular $2\to3$ and $3\to4$ settings studied here.

Exact polystability has a further consequence in these settings.
For uniform vertex and edge dimensions $p<q$, every global section
is supported on vertices whose incident restriction maps all vanish.
Consequently, $H^0(G;\mathcal F)=0$ if each vertex has a nonzero
incident map. Indeed, in balancing inner products, a section
induces a subrepresentation supported on vertex and edge sets
$S$ and $T$. Semistability and the traced balance equations yield
\[
\frac qp\sum_{e\in T}\theta_e
\le -\sum_{v\in S}\theta_v
\le \sum_{e\in T}\theta_e.
\]
Since $\theta_e\ge0$ and $q>p$, both sums vanish. Hence
$\theta_v=0$ on $S$, and the vertex balance equations force
all incident maps there to vanish.
This characterizes the exact balance target and its fixed linear
diffusion limit; finite moment penalties remain soft geometric
biases during nonlinear NSD training.

\section{Empirical Study of Learned Sheaf Geometry}
\label{sec:experiments}

Our experiments investigate the proposed connection between representation structure, stability and geometric control during learning. We first examine the architectural obstruction, then study controlled trivial-component formation and recovery from collapsed representations, and finally investigate how regularization strength influences native NSD geometries. Classification comparisons use paired initializations and CE-only controls; validation-selected checkpoints and fixed training endpoints are evaluated separately.

The native studies use feature-conditioned NSD architectures based on \citet{bodnar2022neural}. Recovery experiments use 16 NSD layers with MLP restriction learners and $3\to4$ stalks, while the regularization-strength and learner-capacity studies use $2\to3$ stalks. The reported WebKB classification comparisons use one fixed public split per graph and three seeds per condition; their CE-only baselines use neither dropout nor weight decay. Training and diagnostic details are in Appendix~\ref{app:experiments}; code, registered configurations and reproduction instructions are available at \url{https://github.com/ArifDoenmez/NSD-GIT}.

\textbf{Rectangular Stalks Enable Stability-Wall Access.}
The equal-stalk obstruction has a direct numerical counterpart. Starting from nonzero admissible weights on the wall $\theta(\mathcal{F}_{\mathrm{triv}})=0$, square $3\to3$ models remain on the wall in all three seeds, whereas rectangular $3\to2$ models leave it in all three. The earlier rectangular controls likewise yield nonzero weights at all ten selected checkpoints, with positive signs under our convention. These observations illustrate the architectural accessibility of nonzero weights.

Negative weights are also compatible with unconstrained sign optimization: six auxiliary $2\to3$ runs on Texas and Wisconsin have negative weights at their validation-selected checkpoints. These runs already initialize negatively, so this is not evidence of spontaneous crossing into the negative region. Subsequent negative-ThetaMM experiments explicitly maintain the sign condition. Architectural accessibility and enforced negativity are therefore distinguished throughout the analysis.

\textbf{Trivial-Component Formation under Dirichlet Minimization.} With node signals fixed and total map norm constrained, direct optimization of the incidence matrices produces verified trivial direct summands in all ten synthetic cases and all six Texas/Wisconsin cases. Explicit inclusion and retraction witnesses establish splitting. These controlled matrix-optimization tests complement the native feature-conditioned classification experiments below.

\textbf{Recovery under Increased Task Load.}
We next initialize a native NSD model with a verified trivial direct summand and identical hidden node signals inside that copy. The encoder remains trainable and retains access to informative raw features. This creates a controlled collapsed starting representation from which both regularized and unregularized models can recover.

For each seed, we hold a 240-node, 4-regular synthetic graph and its training/validation/test split fixed while increasing the task load from three to six and twelve classes. The easier settings show near-perfect classification and no consistent benefit from regularization. At twelve classes, the distinction becomes pronounced: final mean test accuracy is $29.44\%$ for CE-only training, compared with $65.00\%$ and $76.67\%$ for ThetaMM weights one and ten, respectively. All models achieve $100\%$ training accuracy at this endpoint. The improved test accuracy is accompanied by moment-map blocks that are substantially closer to scalar multiples of the identity, even after accounting for their overall scale (relative central imbalance: $0.80$--$0.82$ without regularization versus $0.16$--$0.22$ with regularization).

This supports a geometric regularization effect on generalization under increased task load. It does not isolate class count as the sole cause, since the fixed sample budget also reduces examples per class. In a separate three-class experiment, an input-ablated control remains at chance, confirming the role of informative inputs in recovery in that setting.

\textbf{Recovery on WebKB Graphs.}
We adapt the controlled collapsed initialization to Texas and Wisconsin while retaining their original features, labels and data splits. Both datasets have five classes, compared with twelve in the most demanding synthetic setting. These experiments therefore examine recovery on real heterophilic graphs at a lower class count. At the fixed final endpoint, all models achieve $100\%$ training accuracy, while ThetaMM with weight ten improves test accuracy over its paired CE-only baseline in all six runs. Table~\ref{tab:recovery_results} reports both validation-selected and final results.

\begin{table}[t]
\centering
\small
\begin{tabular}{llccc}
\toprule
Setting & Endpoint & CE-only & ThetaMM, $\lambda_\mu=1$ & ThetaMM, $\lambda_\mu=10$ \\
\midrule
Synthetic, 12 classes
& Selected & $28.89\pm5.85$ & $60.56\pm15.49$ & $74.44\pm1.92$ \\
& Final & $29.44\pm8.22$ & $65.00\pm7.26$ & $76.67\pm5.00$ \\
\midrule
Texas
& Selected & $74.77\pm1.56$ & $77.48\pm1.56$ & $83.78\pm4.68$ \\
& Final & $68.47\pm1.56$ & $79.28\pm1.56$ & $77.48\pm7.80$ \\
Wisconsin
& Selected & $58.82\pm3.40$ & $60.13\pm4.08$ & $60.78\pm1.96$ \\
& Final & $56.21\pm5.99$ & $62.75\pm5.19$ & $69.28\pm2.26$ \\
\bottomrule
\end{tabular}
\caption{Recovery from controlled collapsed initializations: test accuracy (\%, mean $\pm$ sample standard deviation across three paired seeds). Selected checkpoints minimize validation cross-entropy; final checkpoints use 2,000 updates. All final models achieve $100\%$ training accuracy.}
\label{tab:recovery_results}
\end{table}

The generalization benefits remain setting- and metric-dependent: the validation-selected Wisconsin accuracy gains, for example, are not accompanied by improved test macro-recall. Moreover, these task differences do not establish removal of a current trivial component as their cause. They demonstrate that the proposed geometric intervention can influence learning and generalization from a known collapsed starting configuration.

\textbf{Regularization Strength Steers Learned Geometry.}
Finally, we compare CentMM and ThetaMM at weights $\lambda_\mu\in\{0.1,1,10\}$ in native feature-conditioned NSD with 16 and 128 layers. CentMM improves relative central balance at every tested weight in the 16-layer setting. At weight ten, the final relative central residual decreases from $0.547$ to $0.218$ on Texas and from $0.626$ to $0.249$ on Wisconsin. The dependence on weight is not uniformly monotone.

Regularization also changes the energy trajectory in a depth-dependent manner. In the affine-plus-$\tanh$ learner study, final energy is lower in all 36 regularized arm--seed contrasts at 16 layers, but higher in 30 of 36 contrasts at 128 layers. These contrasts share their CE-only controls. Thus, in the deeper setting, regularization predominantly produces higher energies than unregularized training, rather than uniformly reducing energy.

% supplementary separate 16-layer MLP study exhibits particularly large energy reductions under ThetaMM: graph-level median energies decrease from $2814.78$ to $21.01$ on Texas and from $1030.88$ to $11.92$ on Wisconsin, corresponding to reductions of $99.25\%$ and $98.84\%$. Energy divided by squared feature norm also decreases in all six paired runs. These observations concern the 16-layer MLP setting; no corresponding 128-layer MLP comparison was performed.

The findings show that moment-map regularization steers balance and energy during classification training, with the energy response depending on depth and learner setting. Since both the geometry and node signals evolve, neither higher nor lower energy alone establishes improved information preservation. The empirical evidence supports geometric control during learning, while the exclusion interpretation remains conditional on the stability conditions developed above.

\section{Related Work}
\label{sec:related_work}

\textbf{Neural Sheaf Architectures.}
Spectral and energy-based analyses relate oversmoothing to feature concentration under repeated propagation \citep{li2018deeper,chen2020measuring}. Cellular sheaves enrich the diffusion geometry through local restriction maps \citep{hansen2019toward,bodnar2022neural}; attention and connection-Laplacian variants develop this further \citep{barbero2022sheafattention,barbero2022connection}. Bodnar et al.'s NSD provides our controlled reference for isolating stalk dimensions and regularization.

Recent methods intervene at different levels: sheaf positional encodings combine structural and semantic information \citep{he2024positional}, Bayesian networks learn distributions over sheaves \citep{gillespie2024bayesian}, nonlinear diffusion learns the interaction inside the sheaf Laplacian \citep{zaghen2024nonlinear}, and joint diffusion co-evolves signals and restriction maps \citep{hernandezcaralt2024joint}. Polynomial NSD learns spectral filters \citep{borgi2025polynomial}, while cooperative sheaves enable directed communication \citep{ribeiro2026cooperative}. These approaches modify encoding, map learning or propagation; we study representation structure and regularize the balance of the restriction maps.

\textbf{Quiver Structure and Harmonic Information.}
\citet{sumray2024quiver} develop quiver Laplacians and compatible sections, explicitly relating sheaves to subdivided-quiver representations. \citet{hume2025sheafification} give a categorical account of higher-order message passing and use direct sums of sheaves and harmonic spaces for separation results. These provide related mathematical foundations for our decomposition-based perspective. Bundle Neural Networks use diffusion on flat vector bundles \citep{bamberger2025bundle}. Their construction highlights the role of local coordinates: algebraic triviality must be distinguished from homogenization of node features in their original coordinates. Accordingly, our degeneration model also depends on the embedding, signals and task.

\textbf{Recent Analyses of Sheaf Diffusion.}
Identity-sheaf baselines achieve competitive results in finite trained networks \citep{hernandezcaralt2026necessity}, emphasizing the gap between idealized heat-flow limits and complete architectures. Deep NSD changes the propagation operator, normalization and gating to sustain depth-dependent signals \citep{bourgerie2026deep}. Subsequent preprints study relative harmonic information through an index-theoretic criterion \citep{dong2026index} and local operator structure through a Bochner--Weitzenb\"ock decomposition \citep{guadagnuolo2026bochner}. The latter also studies learned-geometry trajectories. These perspectives complement our subrepresentation-based analysis; neither kernel dimension nor energy alone certifies preservation of discriminative information.

\textbf{Implicit Bias, GIT and Geometric Control.}
Implicit-bias results motivate investigating simple learned structures \citep{rahaman2019spectral,gunasekar2017implicit,shah2020pitfalls,arora2019implicit}, without establishing a universal NSD optimization law. King stability and the moment-map characterization of polystability are established GIT tools \citep{king1994moduli,kempf1979length,mumford1994geometric,kirwan1984cohomology}. Our contribution connects these tools to learned NSD geometries: representation decompositions motivate a targeted exclusion criterion, the equal-stalk obstruction identifies an architectural limitation, and differentiable balance surrogates provide mechanisms for geometric control during learning.

\section{Limitations}
\label{sec:limitations}

While this work provides a representation-geometric framework for studying oversmoothing, we view representation degeneracy as a conceptual perspective rather than a universal explanation. Our decomposition results concern fixed linear sheaf diffusion and motivate investigating the structures learned by feature-conditioned NSD. Stability-based exclusion addresses one specific trivial-component mechanism; the usefulness of the remaining geometry also depends on the learned signals and the downstream task. The equal-stalk obstruction is an algebraic result, while our experiments illustrate its architectural implications and the effects of geometric regularization during learning. The controlled synthetic and WebKB studies support moment-map penalties as structured, setting-dependent geometric biases; broader datasets and architectural comparisons remain directions for future work. Finally, our differentiable penalties encourage balanced moment-map configurations but do not enforce exact polystability during gradient descent. For the rectangular $p<q$ architectures, exact polystability forces the global-section space to vanish when every vertex has a nonzero incident map, so fixed linear diffusion then converges to zero at infinite time. This motivates exploring intermediate regularization strengths that balance geometric control with sufficient freedom for task-adapted learning. The associated exclusion interpretation remains conditional on the corresponding stability conditions.

\section{Conclusion}
\label{sec:conclusion}

We connect learned NSD geometries with Geometric Invariant Theory through their interpretation as incidence-quiver representations. Relating representation decompositions to the harmonic space of fixed linear diffusion motivates studying trivial or low-complexity components as potential degeneration modes. King stability supplies a conditional exclusion criterion, and the moment-map characterization of polystability motivates differentiable balance regularizers. We identify the exact equal-stalk obstruction to assigning a destabilizing weight to the trivial representation and show how nonuniform stalk dimensions remove it.

Controlled energy minimization reproduces trivial-component formation, while native NSD studies show that regularization steers balance and energy during learning, with depth-dependent responses and setting-dependent generalization benefits in controlled recovery tasks. These findings illustrate geometric control without establishing universal exclusion or an anti-oversmoothing guarantee. The main contribution is a framework connecting representation structure, stability and trainable geometric interventions, opening further study of how learned sheaf geometries support task-relevant information.

\section*{Acknowledgements and Disclosures}

Arif D\"onmez thanks his wife and children for their support and inspiration. % He is grateful to his academic father Markus Reineke and his academic grandfather Klaus Bongartz. Bongartz sparked his fascination with the representation theory of finite-dimensional algebras during his undergraduate studies, while Reineke deepened his understanding of its geometric aspects throughout his graduate and doctoral studies. This intellectual lineage has shaped the perspective developed in this work.

\textbf{Funding.} This work was supported through the authors' affiliated institutions. No additional project-specific funding was received.

\textbf{Competing interests.} The authors are affiliated with DNTOX GmbH. They declare no other competing financial interests.

\bibliographystyle{plainnat}
\bibliography{refs}

\appendix

% Compact supplementary support for Section 7; full records accompany the PDF.
\clearpage
\section{Experimental Details}
\label{app:experiments}

\textbf{Data and comparisons.}
The current WebKB studies use frozen Texas and Wisconsin split-0 inputs from Geom-GCN \citep{pei2020geom}. Features are row-normalized nonnegative bag-of-words vectors. Removing self-loops and duplicate edge directions gives 183 nodes/279 undirected edges for Texas and 251/450 for Wisconsin, with 1,703 features and five classes each. Actual train/validation/test counts are 87/59/37 and 120/80/51. Masks are retained unchanged; one Texas class is absent from its training and test masks. Macro-recall averages over classes present in the evaluated mask. Features and graph structure are available transductively; task gradients use training labels only.

\begin{table}[htbp]
\centering
\small
\setlength{\tabcolsep}{4pt}
\begin{tabular}{llllll}
\toprule
Study & Stalks & Layers & Updates & Optimizer / LR & Seeds \\
\midrule
Common-wall pilot & $3\to3$, $3\to2$ & 4 & $\leq1500$ & Adam / $10^{-2}$ & 3 \\
Free-sign pilot & $2\to3$ & linear diffusion & 600 & Adam / $10^{-2}$ & 3/graph \\
Dirichlet, $K_4$ & $2\to3$ & -- & 5000 & SGD / $0.005$ & 10 \\
Dirichlet, WebKB & $2\to3$ & -- & 5000 & SGD / $0.005$ & 3/graph \\
Native weight study & $2\to3$ & 16, 128 & 2000 & Adam / $10^{-4}$ & 3/graph \\
Recovery studies & $3\to4$ & 16 & 2000 & Adam / $10^{-3}$ & 3/setting \\
\bottomrule
\end{tabular}
\caption{Protocols underlying the main empirical results. Recovery includes the initial three-class control, the class-load study and WebKB adaptation. Fixed training endpoints and validation-selected checkpoints are distinguished throughout.}
\label{tab:supp_protocols}
\end{table}

\textbf{Native architecture and objectives.}
We use a local feature-conditioned General-NSD implementation based on \citet{bodnar2022neural}, with 20 channels, learned stalk/channel mixing, residual coefficients and propagation ELU. Each layer has its own restriction learner, shared across incidences. The weight study uses an ordered-endpoint linear learner followed by $\tanh$; recovery uses a $120\to128\to12$ ReLU MLP with output $\tanh$. Each output matrix is capped at Frobenius norm one. Native weight and recovery comparisons have no dropout, weight decay, clipping, schedule or early stopping. The separate common-wall pilot uses dropout $0.7$, weight decay $0.005$, gradient clipping at one and validation patience 200. The free-sign pilot generates maps once per forward and holds them fixed through linear diffusion; it does not measure nonlinear layer depth.

\label{app:losses}
Native ThetaMM uses Smooth-L1/Huber loss with transition one, averaging residual entries separately over vertex and edge blocks, summing those group means and averaging over layers. Its objectwise $\theta$ is shared across layers, with mean-squared norm penalty $10^{-4}$. CentMM uses the same averaging with squared residuals and per-layer analytic targets $\widehat\theta_i=\operatorname{tr}(\mu_i)/d_i$, without learned $\theta$ or a $\theta$-norm penalty. Doses are $\lambda_\mu=0.1,1,10$ in the weight study and one/ten in recovery. Equal weights across Huber and quadratic losses are not matched penalty strengths. Negative-ThetaMM enforces $\theta\cdot\mathbf d=0$ and $\sum_i\theta_i\leq-0.01$ after updates; the free-sign pilot retains only admissibility after its negative initialization.

\textbf{Controlled energy and collapsed initialization.}
Direct energy tests optimize free incidence matrices with fixed node signals and total squared map norm one. On $K_4$, $x_v=e_1$; on WebKB, fixed unit signals are obtained without labels from the input features. These calibrations use neither classification loss nor native end-to-end map learning. Their regularized controls use quadratic moment residuals; exact settings and tolerance-sensitive outcomes are retained in the accompanying protocols.

Recovery instead begins with native maps $A_{v,e}=\tfrac12uv^\top$ for unit $u\in\mathbb R^4$, $v\in\mathbb R^3$. The vectors $s_v=v$, $s_e=\tfrac12u$ and covectors $p_v=v^\top$, $p_e=2u^\top$ verify a trivial direct summand, leaving complementary dimensions two and three. The encoder has zero weight and bias $vc^\top$ for 20 positive channel coefficients. All parameters remain trainable, with informative raw inputs retained. On irregular WebKB graphs, channel propagation weights additionally start at zero to preserve identical hidden states. Initial raw and propagated-input energies are zero; normalized hidden energy need not be zero.

Synthetic recovery uses 4-regular circulants with offsets one and seven, randomly relabeled within each seed. The class-load study fixes 240 nodes and 120/60/60 masks across $K=3,6,12$; labels before relabeling are index modulo $K$. Orthonormal class means have amplitude $1.5$ in 12 input dimensions, with Gaussian noise standard deviation $0.5$ and train-only centering/scalar-RMS preprocessing. Training examples per class decrease from 40 to 20 to 10. The separate three-class control uses 120 nodes and 60/30/30 masks and includes a zero-input ablation.

\textbf{Measurements and numerical verification.}
For the actual layer input $H_\ell$, the reported network energy and common central imbalance are
\[
E_\ell=\|\delta_{\mathcal F^{(\ell)}}N_\ell H_\ell\|_F^2,
\quad (N_\ell)_v=\left(I+\sum_e A_{v,e}^*A_{v,e}\right)^{-1/2},
\quad b_{\rm cent}=\sqrt{\frac{\sum_i\|\mu_i-\widehat\theta_i I\|_F^2}{\sum_i\|\mu_i\|_F^2}}.
\]
The sum in $N_v$ is over incidences at $v$; inverse square roots use 20 Newton--Schulz iterations. Energy omits a factor $1/2$. Summaries take a layer median per model and then a median across seeds. Relative imbalance is distinct from the optimized averaged Huber term. Splitting requires compatible inclusions and retractions, $A_{v,e}s_v=s_e$, $p_eA_{v,e}=p_v$ and $p_is_i=1$ on every object. Relative SVD cutoffs are $10^{-10}$, $10^{-8}$ and $10^{-6}$, with witness equation errors at most $10^{-8}$. Coarser thresholds recover approximate witnesses in two strongly regularized synthetic endpoints; numerical non-detection is not an exact exclusion theorem.

\textbf{Selection and reproducibility.}
Dose checkpoints are evaluated initially and every 25 updates; recovery uses every five updates through 100 and every 25 thereafter. Minimum validation CE selects checkpoints, with earliest ties, before cohort test evaluation. Fixed final endpoints remain separately reported. CPU float64 and one PyTorch/BLAS thread per worker are used for the current native and direct-energy studies; the common-wall pilot uses MPS float32. Earlier ten-split weights are archival sign audits, not current classification replications. Studies were fixed before their own execution; follow-ups used knowledge of preceding results. We organize model variants, hypotheses and comparison configurations with ArchTrial \citep{doenmez2026archtrial}, a command-line tool for structuring architecture experiments. Our project-specific runners additionally record paired initialization checks and checkpoint-selection locks. Full protocols, seed lists and individual endpoints are available in the public code repository linked in Section~\ref{sec:experiments}; layers and shared-control contrasts are not independent replications.

\section{Additional Results for Central Moment Regularization}
\label{app:centmm}

CentMM provides a geometric baseline that determines its scalar targets analytically rather than learning stability weights. Table~\ref{tab:supp_centmm} summarizes its effect in the verified native weight study. At 16 layers, every tested CentMM weight improves relative central balance against CE-only on both graphs. Increasing the weight is not uniformly monotone: Wisconsin reaches its lowest residual at weight one. At 128 layers, the changes are smaller and not uniformly favorable at lower weights. This complements ThetaMM by showing that geometric control can already arise from enforcing the central form of moment blocks without explicit stability-weight optimization.

\begin{table}[htbp]
\centering
\small
\begin{tabular}{lrrrrr}
\toprule
Graph & Layers & CE-only & CentMM 0.1 & CentMM 1 & CentMM 10 \\
\midrule
Texas & 16 & 0.547 & 0.254 & 0.235 & 0.218 \\
Texas & 128 & 0.585 & 0.594 & 0.585 & 0.560 \\
Wisconsin & 16 & 0.626 & 0.315 & 0.227 & 0.249 \\
Wisconsin & 128 & 0.612 & 0.589 & 0.576 & 0.563 \\
\bottomrule
\end{tabular}
\caption{Final relative central imbalance $b_{\rm cent}$ in the native weight study. Each entry is a median over three seeds of the per-model layer medians. The numbers after CentMM denote $\lambda_\mu$. Lower values indicate moment blocks closer to scalar matrices relative to their overall scale.}
\label{tab:supp_centmm}
\end{table}

These measurements concern relative balance, not exact polystability or component removal. The additional registered all-layer numerical exclusion margin was not achieved in the native dose cohort. Complete task endpoints, including unfavorable outcomes and both checkpoint policies, remain in the machine-readable records. The separate learner-capacity and scale audits are also retained there without repeating their full tables or trajectories in this PDF.

%\newpage
%\input{checklist.tex}

\end{document}